\RequirePackage{fix-cm}
\documentclass{aptpub}

\authornames{B Goodman, and PF Tupper} 
\shorttitle{Effects of Limiting Memory Capacity \\ on the Behaviour of Exemplar Dynamics} 

\usepackage{amsmath}
\usepackage{amsfonts}
\usepackage{graphicx}
\usepackage{bm}
\usepackage{bbm}
\usepackage[bookmarks=false]{hyperref}
\usepackage{setspace}
\usepackage{subfig}
\usepackage[justification=raggedright,singlelinecheck=false]{caption}

\newcommand{\R}{\mathbb{R}}

\newcommand{\NN}{\mathbb{N}}

\newcommand{\beq}{\begin{eqnarray*}}
\newcommand{\eeq}{\end{eqnarray*}}
\newcommand{\beqal}{\begin{align*}}
\newcommand{\eeqal}{\end{align*}}
\newcommand{\beqa}{\begin{eqnarray}}
\newcommand{\eeqa}{\end{eqnarray}}
\newcommand{\beqaal}{\begin{align}}
\newcommand{\eeqaal}{\end{align}}
\newcommand{\eps}{\varepsilon}

\newcommand{\expect}{\mathop{\mathbb{E}}}
\newcommand{\pr}{\bold{P}}
\newcommand{\ind}{\mathbbm{1}}
\newcommand{\Var}{\mathrm{Var}}

\newcommand\numberthis{\addtocounter{equation}{1}\tag{\theequation}}

\begin{document}

\title{Effects of Limiting Memory Capacity \\ on the Behaviour of Exemplar Dynamics}

%



\authorone[Simon Fraser University]{B Goodman}

\authortwo[Simon Fraser University]{PF Tupper}
\address{Department of Mathematics, Simon Fraser University, 8888 University Dr., Burnaby, BC, V5A 1S6, Canada}

\emailone{bgoodman@sfu.ca}
\emailtwo{pft3@sfu.ca}





\begin{abstract}
Exemplar models are a popular class of models used to describe language change.  Here we study how limiting the memory capacity of an individual in these models affects the system's behaviour.  In particular we demonstrate the effect this change has on the extinction of categories.  Previous work in exemplar dynamics has not addressed this question.  In order to investigate this, we will inspect a simplified exemplar model.  We will prove for the simplified model that all the sound categories but one will always become extinct, whether memory storage is limited or not.  However, computer simulations show that changing the number of stored memories alters how fast categories become extinct.
\keywords{Exemplar Models; Linguistics; Language Change; Extinction}

\ams{91F20}{70F99}
\end{abstract}

\section{Introduction}
\label{sec:introduction}

In spoken and written language, there are instances where there are two or more variants of a word, each of which is equivalent from the point of view of communication.  We can think of instances of the word as belonging to one of two or more categories.  For example, a population might pronounce the word ``either'' as both ``ee-ther'' and ``eye-ther''.  Another example is when there are different spellings of a word.  Figure \ref{fig:cider}, which was generated by Google Ngram Viewer \cite{ngramviewer}, shows in the written lexicon a comparison between the usage of the word ``cider'' and its archaic spelling ``cyder''.  In the year $1800$ ``cyder'' seems to have been the more popular spelling but it has become practically extinct since then.  As we see in this example, it is possible for a category to become extinct, passing out of usage.

\begin{figure}[h!]
\includegraphics[width=14cm]{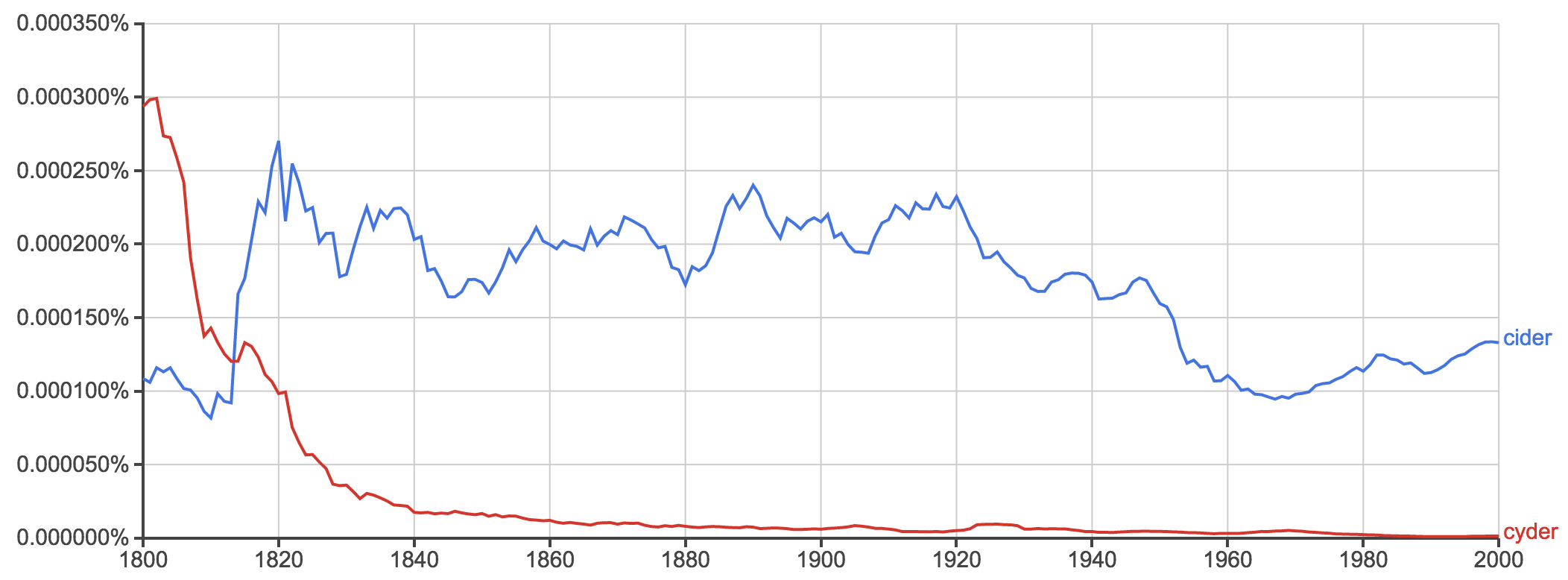}
	\caption{Comparison of the usage of ``cider'' (blue) and its archaic spelling ``cyder'' (red) within a corpus of books between the years $1800$ and $2000$.  The $y$-axis represents the percentage of the usages of the words in the entire database.  This image was generated by Google Ngram Viewer \cite{ngramviewer}.}
	\label{fig:cider}
\end{figure}

Here we study a model for just this kind of category extinction.  One popular class of models used to research the evolution of spoken and written language are exemplar models which was first introduced by Nosofsky \cite{nosofsky1986,nosofsky1988}.  Nosofsky hypothesized that people store detailed memories of stimuli they are exposed to which are called exemplars \cite{tupper2015}.  Work done by Johnson \cite{johnson1997} showed exemplar theory could be applied to model speech perception.

Exemplar theory models language use in one individual.  Exemplars are detailed memories of utterances of sounds, each with its own category label.  Categories are formed of all exemplars with a given category label.  Exemplars are represented as vectors where each dimension represents a phonetic variable such as fundamental frequency or tongue height.  Each exemplar will have a weight (or activation) associated with it, representing how predominant or recent the memory of the sound is.  In many exemplar models these weights decay exponentially over time \cite{pierrehumbert}.

Exemplar dynamics builds on exemplar theory by creating a production-perception loop between two individuals with their own stored exemplars.  Exemplar dynamics was first used to model speech production and perception by Pierrehumbert in \cite{pierrehumbert}.  Ever since, many linguists have used exemplar dynamics to model spoken and written language such as \cite{jaeger1,wedel2012,tupper2015,bybee2002,winter2016} to list a few.

In exemplar dynamics, there are usually two individuals speaking to one another.  Each individual has a store of labeled exemplars.  At every time step a new sound is produced by a speaker, which is then perceived by the listener and classified based on the listener's stored exemplars.  The way the sound is produced varies depending on the model.  Usually a new sound is produced randomly by adding noise and bias to a pre-existing exemplar.  The listener usually categorizes sounds based on their `closeness' to the cloud of exemplars stored for each category.  The weights of the exemplars decay at each time step, and the process is repeated.  Newly categorized sounds become a part of the perception process, continually evolving the system \cite{pierrehumbert}.

The extinction of a category occurs when the weights for all the exemplars labelled in that category approach zero.  This represents the listener no longer remembering the category.  The listener will cease to produce tokens from that category.  A necessary condition for the extinction of a category is that the probability of classifying a sound as that category must approach zero.  In this paper we will be particularly interested in when there is extinction of all but one category.

This paper is motivated by research in exemplar dynamics done by Tupper \cite{tupper2015} and Wedel \cite{wedel2012}.  They both studied the same exemplar dynamic model, but with a subtle difference.  In \cite{wedel2012} categories were limited to a maximum of $100$ stored exemplars, whereas in \cite{tupper2015} categories had no limitation on the number of stored exemplars.  It was demonstrated in \cite{tupper2015} that when the number of exemplars stored is unlimited, then there is extinction of all but one category.  In \cite{wedel2012}, it was observed that there is no category extinction in simulations for a certain choice of parameters when the exemplars stored per category is limited to $100$.  However, in \cite{wedel2012}, Wedel only did numerical simulations of his model up to $4000$ iterations.

This begs the question, when you limit the number of exemplars to be stored per category, will categories eventually become extinct?  In this paper we seek the answer to this question.  The models of \cite{tupper2015} and \cite{wedel2012} are too complicated to investigate rigorously, so we study a simpler model which captures some of their essential features.


In Section \ref{sec:general}, we describe our simple exemplar model.  Our model only depends on three parameters: the number of categories $k$, the decay rate $\lambda$, and the number of exemplars stored per category $N$.  Two particular cases of this general model will be studied: one where we limit the number of exemplars ($N<\infty$, as in \cite{wedel2012}) in Section \ref{singlestored}, and another where we do not ($N=\infty$, as in \cite{tupper2015}) in Section~\ref{infstored}.  We prove in both cases that all categories but one will become extinct.  In Section~\ref{simstoptime} we discuss computational results, which demonstrates how limiting the number of exemplars affects the system's evolution.  The numerical simulations in this section will help us explain the effect $N$ and $\lambda$ have on the expected time to extinction.

\section{Simple Exemplar Weight Model}
\label{sec:general}

In this section we describe a simplified exemplar model.  The parameters for the system are the number of categories, $k$, the number of exemplars stored per category, $N$, and the decay rate, $\lambda$.  The listener starts with some exemplars with associated weights in each category, and then receives a stream of new inputs (sounds).  The listener in this model will decide how to classify new sounds only using the total weights of the exemplars in each category.  The phonetic information stored in exemplars will not be utilized in the categorization process.


At time $n$, let  $w_{j,m}^{n}$ be the weight of the $m$th exemplar where $m\in \NN$, for category $j$.  At time $n$, these $k$ infinite sequences of real numbers comprise the state of the system.  Note that throughout this paper, superscript $n$ is an index referring to time $n$, and not the exponent $n$.  Let $N$ be the maximum number of exemplars per category the listener is permitted to store.  Let $\lambda >0$ be the decay rate of the weights, so that at each time step $n$, the weights of old memories will decay by a factor of $ \beta =e^{-\lambda }$.  New exemplars are given a weight $W_{0}=1$.  Additionally, when $N<\infty$, if there are $N+1$ exemplars in a category with non-zero weight upon adding a new exemplar, then the exemplar with the lowest weight is discarded.

We assume that exemplars are ordered by weight at all times, so that $0\leq w_{j,m+1}^{n} \leq w_{j,m}^{n} \leq 1$, for all $n\in \NN_{0}$, $j\in \{ 1,\hdots,k \}$, and $m< N$.  The initial conditions of the weights are non-random, and can be anything such that $0\leq  w_{j,m}^{0} \leq 1$, for all $j\in \{ 1,\hdots,k \}$, and $m\leq N$, at least one of the weights in one category must be non-zero, and if $N<\infty$, then $w_{j,m}^{0}=0$ for all $j\in \{ 1,\hdots,k \}$, and $m>N$.

Let $W_{j}^{n}:=\sum_{m=1}^{N}w_{j,m}^{n}$ be the total weight of exemplars in category $j\in \{ 1,\hdots,k \}$, and $W_{\mbox{\textit{tot}}}^{n}:=\sum_{j=1}^{k} W_{j}^{n}$ be the total weight of all exemplars.


Let $x_{n}$ be the category we classify the $n$th sound as at time $n$.  For example, $x_{n}=j$ means we classified the $n$th sound as category~$j$.  We let the probability of classifying the $n$th sound as category $j$ ($x_{n}=j$), given the state of the system in the previous time step be $W_{j}^{n}  / W_{\mbox{\textit{tot}}}^{n}$.  This classification procedure is the Luce choice rule \cite{luce}.  As such, the categorization of sounds only depends on the weights of the exemplars, unlike other models where the phonetic information stored in exemplars is used to classify sounds.

To aid in the analysis of our model, we define a filtration that the processes $\{w_{j,m}^n\}_{n \geq 0}$ are adapted to. First, let $\mathcal{F}$ be the $\sigma$-field generated by all random variables in the model.
We then define the sequence  of $\sigma$-fields $\mathcal{F}_{n}$  for $n \geq 0$ by
\beqa
\mathcal{F}_{n}=\sigma (w_{j,m}^{q},~0<q\leq n,~j\in\{1,\hdots,k\},~m\in \NN). \label{sigmafield}
\eeqa
 This sequence of $\sigma$-fields forms a filtration since, for all $n$, $\mathcal{F}_{n} \subset \mathcal{F}_{n+1} \subset \mathcal{F}$ \cite[pg.458]{billingsley1995}.


Another way to describe the process is the following: At time step $n$, ${\bold{P}(x_{n}=j|\mathcal{F}_{n})=  W_{j}^{n}  / W_{\mbox{\textit{tot}}}^{n} }$, for each $j\in \{1,\hdots , k\}$.  If $x_{n}=j$, then:
\begin{enumerate}
\item Let $w_{j,m+1}^{n+1}= \beta w_{j,m}^{n}$, for all $m< N$, and ${w_{j,1}^{n+1}=W_{0}=1}$.  If $N<\infty$, the exemplar corresponding to the $N$th position of category $j$ from the previous time step will be discarded: Thus, if $N< \infty$, we let $w_{j,N+1}^{n+1}=0$.
\item For all $i\neq j$, and $m\in \{1,\hdots, N \}$, let $w_{i,m}^{n+1}= \beta w_{i,m}^{n}$.
\end{enumerate}

The next couple of sections are devoted to proving the model just described always results in the extinction of all but one category.  Sections \ref{singlestored} and \ref{infstored} will respectively look at the cases where $N<\infty$ and $N=\infty$.

\section{Finite Stored Exemplars Model}
\label{singlestored}

In this section we will show that when $N<\infty$, all but one category will become extinct with a probability of $1$.  That is, it will be proved that with a probability of $1$, there exists an $M$ and a $j$, such that $x_{n}=j$, for all $n\geq M$.

The following lemma proves if one classifies $p$ consecutive sounds as category $j$, then it only increases the probability of the next sound being classified as category $j$.

\begin{lemma}
If $N<\infty$, and $\mathcal{F}_{n}$ is the $\sigma$-field defined by Equation \ref{sigmafield}, then
\beq
\bold{P}(x_{n+p}=j | x_{n+p-1}=j, \hdots , x_{n}=j , \mathcal{F}_{n})\geq \bold{P}(x_{n}=j | \mathcal{F}_{n}),
\eeq
a.s., for all $p\in \NN_{0}$, and $j\in \{ 1, \hdots , k \}$.
\label{increaselemma}
\end{lemma}

\begin{proof}



We will prove this lemma using induction.  Let $S(p)$ be the statement that
\beq
\bold{P}(x_{n+p}=j | x_{n+p-1}=j, \hdots , x_{n}=j , \mathcal{F}_{n})\geq \bold{P}(x_{n}=j | \mathcal{F}_{n}),
\eeq
a.s.  We want to prove $S(p)$ is true for all $p\in \NN_{0}$.

The initial statement $S(0)$, which is $\bold{P}(x_{n}=j |  \mathcal{F}_{n})\geq \bold{P}(x_{n}=j | \mathcal{F}_{n})$, is true because the two sides are equal.



Now we assume the inductive hypothesis $S(p)$ is true; $\bold{P}(x_{n+p}=j | x_{n+p-1}=j, \hdots , x_{n}=j , \mathcal{F}_{n})\geq \bold{P}(x_{n}=j | \mathcal{F}_{n})$.  We want to show $S(p+1)$ is true.  If $\{ x_{n+p}=j, \hdots , x_{n}=j \}$, then $W_{j}^{n+p+1}=\beta W_{j}^{n+p}+1-\beta w_{j,N}^{n+p}$, and $W_{ \mbox{\textit{tot}} }^{n+p+1}=\beta W_{ \mbox{\textit{tot}} }^{n+p}+1-\beta w_{ j ,N}^{n+p}$.  One can show via simple algebra, using the facts that $\beta^{-1}(1-\beta w_{j,N}^{n+p})>0$, and $W_{j}^{n+p} \leq W_{ \mbox{\textit{tot}} }^{n+p}$, that
\begin{align*}
\bold{P}(x_{n+p+1}=j | x_{n+p}=j, \hdots , x_{n}=j , \mathcal{F}_{n})&=\dfrac{ W_{j}^{n+p}+ \beta^{-1}(1-\beta w_{j,N}^{n+p}) }{ W_{ \mbox{\textit{tot}} }^{n+p}+\beta^{-1}(1-\beta w_{ j ,N}^{n+p}) } \\
&\geq \dfrac{ W_{j}^{n+p}}{ W_{ \mbox{\textit{tot}} }^{n+p} } \numberthis \label{inductioneqn}
\end{align*}
a.s.  The right hand side of Equation \ref{inductioneqn} is equal to $\bold{P}(x_{n+p}=j | x_{n+p-1}=j, \hdots , x_{n}=j , \mathcal{F}_{n})$, which implies by the induction hypothesis that statement $S(p+1)$ is true. \end{proof}

Define $A_{n}$ to be the event that we only classify input sounds as a single category from time step $n$ onwards.  More precisely, 
\beqa
A_{n}=\{ \exists j, x_{m}=j, \forall m \geq n  \}.
\label{AnEqn}
\eeqa
The event $A_{n}$ will be important throughout this section.

\begin{lemma}  If $N <\infty$, and $\mathcal{F}_{n}$ is the $\sigma$-field defined by Equation \ref{sigmafield}, then there exists a $Q>0$, such that $\bold{P}(A_{n}|\mathcal{F}_{n})\geq Q$, for all $n$, where $A_{n}$ is the event defined by Equation \ref{AnEqn}.
\label{pastlemma}
\end{lemma}

\begin{proof}

At time step $n$, there must exist a category $c\in \{ 1, \hdots , k \}$, such that $\bold{P}(x_{n}=c |\mathcal{F}_{n}) \geq k^{-1}$.  By Lemma \ref{increaselemma}, we get the following inequality,
\begin{align*}
\bold{P}( x_{n+N-1}=c,\hdots, x_{n+1}=c, x_{n}=c | \mathcal{F}_{n}) \\
=\prod_{p=0}^{N-1} \bold{P}( x_{n+p}=c | x_{n+p-1}=c,\hdots , x_{n}=c , \mathcal{F}_{n}) \geq k^{-N}.\numberthis \label{kineq}
\end{align*}
If $x_{q}=c$, for $n\leq q \leq n+N-1$, then $W_{c}^{n+N}=\sum_{q=0}^{N-1}\beta^{q}$, and if we continue to categorize $x_{q}=c$, for $q>n+N-1$, the weight for category $c$ will stay constant.  Let $\Omega=\sum_{q=0}^{N-1}\beta^{q}$.  Upon inspection, it is apparent that $W_{i}^{n}\leq \Omega$, for all $i$ and $n >N-1$, because it is the maximum total weight a category can have after there have been at least $N$ time steps.

If $x_{p}=c$, for $n\leq p \leq n+q$, where $q\geq N-1$, then
\begin{align*}
 W_{\mbox{\textit{tot}}}^{n+q}&=\sum_{p=1}^{k}W_{p}^{n+q} =W_{c}^{n+q}+\sum_{p\neq c}W_{p}^{n+q} \leq \Omega+ \sum_{p\neq c}\beta^{q} \Omega < \Omega(1+k \beta^{q}) .
\end{align*}

Let $G_{n,N}=\{ x_{n+N-1}=c , \hdots , x_{n+1}=c, x_{n}=c \}$.  The probability of categorizing the next sound as $c$, given that we have only categorized as $c$ since time step $n$, and have at least done so $N$ times in a row can be bounded below, 
\begin{align*}
\bold{P}(x_{n+N-1+q}=c | x_{n+N-2+q}=c,\hdots, x_{n+N}=c, G_{n,N} , \mathcal{F}_{n}) &= \dfrac{\Omega}{W_{\mbox{\textit{tot}}}^{n+N-1+q}} \\
&\geq \dfrac{\Omega}{\Omega ( 1+k \beta^{q} ) } \\
&= 1-\dfrac{k \beta^{q} }{ 1+k \beta^{q} },  \numberthis \label{multineq}
\end{align*}
for all $n,q,N>0$.  Note we used the fact that $W_{j}^{n+N-1+q}=\Omega$, because the event $G_{n,N}$ had already occurred.

Utilizing Equations \ref{kineq} and \ref{multineq},
{\allowdisplaybreaks
\begin{align*}
\bold{P}( x_{m}=c ,\forall m\geq n | \mathcal{F}_{n}) = \bold{P}\left( \bigcap_{q=0}^{\infty} \left\{ x_{n+q}=c \right\} \bigg|\mathcal{F}_{n} \right) 
\\
=\bold{P}\left(\bigcap_{q=1}^{\infty} \left\{ x_{n+N-1+q}=c \right\} \bigg| G_{n,N} , \mathcal{F}_{n} \right) \bold{P}(G_{n,N} | \mathcal{F}_{n} )    
 \\
\geq k^{-N} \prod_{q=1}^{\infty}\bold{P}(x_{n+N-1+q}=c | x_{n+N-2+q}=c,\hdots, x_{n+N}=c,G_{n,N} ,\mathcal{F}_{n}) 
\\
\geq  k^{-N} \prod_{q=1}^{\infty} \left( 1-\dfrac{k \beta^{q} }{ 1+k \beta^{q}} \right) . \numberthis \label{eq2}
\end{align*}
}
By Theorem $15.5$ in \cite{rudinanalysis}, the product in Equation \ref{eq2} is strictly greater than $0$ if and only if
\beq
\sum_{q=1}^{\infty}\dfrac{k \beta^{q} }{ 1+k \beta^{q} }<\infty.
\eeq
By the ratio test we know this series is convergent.  Therefore, there is a $Q>0$, such that $\bold{P}( x_{m}=c ,\forall m\geq n|\mathcal{F}_{n})\geq Q >0$.

Since
{\allowdisplaybreaks
\begin{align*}
\bold{P} \left( \exists j, x_{m}=j, \forall m \geq n  | \mathcal{F}_{n} \right) &\geq \bold{P} \left(x_{m}=c, \forall m\geq n | \mathcal{F}_{n}\right)  \geq Q>0,
\end{align*}}
we get the final result. \end{proof}


Lemma \ref{pastlemma} states that the probability of $A_{n}$ (Equation \ref{AnEqn}) occurring, given any event which only depends on the events up to time step $n-1$, can be bounded below by a constant $Q>0$.  In other words, the probability of $x_{m}$ being classified as the same category for all $m\geq n$, always has at least a certain probability of happening no matter what occurs before it.

Note, if $A_{n}$ is true for any value of $n$, the rest of the categories $i\neq j$ will become extinct.  If we prove that $\bold{P}\left( \bigcup_{n=1}^{\infty}A_{n} \right)=1$, then we have proved there is almost surely extinction of all but one category when $N<\infty$.

\begin{lemma}
Let $\mathcal{G}$ be a $\sigma$-field, $G\in \mathcal{G}$, such that $\bold{P}(G)>0$, and $X$ be an event,.  If there exists a $Q>0$, s.t. $\bold{P}(X|\mathcal{G})\geq Q$, a.s., then $\bold{P}(X|G)\geq Q$.
\label{sigmalemma}
\end{lemma}
\begin{proof}
By the definition of the probability of an event conditioned on a {$\sigma$-field} \cite[p.155]{rosenthal},
\begin{align*}
\bold{P}(X|G) &= \dfrac{\bold{P}(X\cap G)}{\bold{P}(G)} =\dfrac{\expect\left[\bold{P}(X|\mathcal{G})\ind_{G} \right]}{\bold{P}(G)} \geq \dfrac{\expect\left[Q \ind_{G} \right]}{\bold{P}(G)} \geq Q.
\end{align*}
\end{proof}

\begin{thm}
When $N<\infty$, all categories but one will become extinct with a probability of $1$: that is, $\bold{P}\left(\bigcup_{n=1}^{\infty}A_{n}\right)=1$, where $A_{n}$ is given by Equation \ref{AnEqn}.
\label{collapsethm}
\end{thm}
\begin{proof}
The proof utilizes Murphy's Law, a general statement proven in \cite{steel2015}.  Murphy's Law states the following:  Let $(G_{n},n\geq 1)$ be any sequence of events satisfying the condition $G_{n}\subseteq G_{n+1}$, for all $n\geq 1$, and let $G=\bigcup_{n=1}^{\infty}G_{n}$.  If $\bold{P}(G|G_{n}^{c})\geq \eps>0$, for all $n\geq 1$, then $\bold{P}(G)=1$.

We know $A_{n}\subseteq A_{n+1}$, for all $n\geq 1$.  Let $A=\bigcup_{n=1}^{\infty}A_{n}$.  By Murphy's law, if we can show $\bold{P}(A|A_{n}^{c})\geq \eps >0$, for all $n$, then $\bold{P}(A)=1$, proving the theorem.

Let $Y_{n}= \min \{m\in \NN : x_{n+m}\neq x_{n} \} $, and if it is not defined then let $Y_{n}=\infty$.  The event $\{Y_{n}=m \}$ is a subset of $A_{n}^{c}=\{  \exists j>n, \mbox{ s.t. } x_{j}\neq x_{n} \}$, for all $n$, and $A_{n}^{c}=\bigcup_{m>0}\{ Y_{n}=m\}$.  Using the fact that the events $\{ Y_{n}=i\}$ and $\{ Y_{n}=j \}$ are disjoint when $i\neq j$, we obtain the following,
{\allowdisplaybreaks
\begin{align*}
\bold{P}(A|A_{n}^{c})&=\dfrac{\bold{P}(A\cap A_{n}^{c})}{\bold{P}( A_{n}^{c} ) } =\dfrac{\bold{P}\left( A\cap \left(\bigcup_{m>0} \{Y_{n}=m \} \right) \right) }{\bold{P}( A_{n}^{c})} \\
&= \dfrac{ \bold{P}\left( \bigcup_{m>0}\{ A \cap \{Y_{n}=m \} \} \right)}{\bold{P}(A_{n}^{c} )} \\
&=\dfrac{ \sum_{m>0}\bold{P}\left(  A \cap \{Y_{n}=m \}   \right) }{\bold{P}( A_{n}^{c})} \\
&\geq  \dfrac{ \sum_{m>0}\bold{P}\left(  A_{n+m+1} \cap \{Y_{n}=m \}  \right) }{\bold{P}( A_{n}^{c})}
\end{align*}
since $A_{n+m+1}\subseteq A$.  Using Lemma \ref{pastlemma} with Lemma \ref{sigmalemma} (noting $\{Y_{n}=m \}$ is in $\mathcal{F}_{n+m}$), and that $\bigcup_{m>0}\{ Y_{n}=m\}=A_{n}^{c}$, we obtain
\begin{align*}
\dfrac{ \sum_{m>0}\bold{P}\left(  A_{n+m+1} \cap \{Y_{n}=m \}   \right) }{\bold{P}( A_{n}^{c})}&=   \sum_{m>0}\bold{P}\left(  A_{n+m+1} | \{Y_{n}=m \}  \right) \dfrac{  \bold{P}\left( Y_{n}=m   \right)}{\bold{P}( A_{n}^{c})} \\
&\geq  Q \sum_{m>0}  \bold{P}\left( Y_{n}=m  | A_{n}^{c}\right) \\
&=  Q \cdot  \bold{P}\left( \bigcup_{m>0}\{ Y_{n}=m\}  \bigg| A_{n}^{c}\right)=Q>0.
\end{align*}  } \end{proof}

\section{Infinite Stored Exemplars Weight Model}
\label{infstored}

This section will be devoted to studying the special case of the model where the listener stores an infinite number of exemplars, so $N=\infty$.  The proof for showing there is almost surely extinction of all but one category in this special case will be different from the previous section.

Let $Z_{j}^{n}=\bold{P}(x_{n}=j | \mathcal{F}_{n})=W_{j}^{n} / W_{\mbox{\textit{tot}}}^{n}$, where $\mathcal{F}_{n}$ is as defined in Equation \ref{sigmafield}.  Note the combined weights of all categories which are not $j$ is equal to ${W_{\mbox{\textit{tot}}}^{n}-W_{j}^{n}}$.  We will first re-describe the model's evolutionary process in terms of $W_{j}^{n}$ and $W_{\mbox{\textit{tot}}}^{n}$, in order to simplify the proof.  The evolutionary process evolves as follows:
\begin{itemize}
\item If $x_{n}=j$, then the total weight of category $j$ becomes ${W_{j}^{n+1}=1+W_{j}^{n} \beta}$, and the total weight of all other categories besides $j$ becomes \\ ${W_{\mbox{\textit{tot}}}^{n+1}-W_{j}^{n+1}=(W_{\mbox{\textit{tot}}}^{n}-W_{j}^{n}) \beta}$.
\item If $x_{n}\neq j$, then the total weight of all categories besides $j$ is \\ ${W_{\mbox{\textit{tot}}}^{n+1}-W_{j}^{n+1}=1+(W_{\mbox{\textit{tot}}}^{n}-W_{j}^{n}) \beta}$, and the total weight of category~$j$ is $W_{j}^{n+1}=W_{j}^{n} \beta$.
\end{itemize}

We want to prove there exists a category $j$, such that $Z_{j}^{n}\rightarrow 1$, a.s., as $n\rightarrow \infty$, and for the rest of the categories $q\neq j$, that $Z_{q}^{n}\rightarrow 0$, a.s.

We want to prove that for all $j$, that $Z_{j}^{n}$ can only converge to $0$ or $1$, a.s.  We will then show that if $Z_{j}^{n}\rightarrow 0$, a.s., then $W_{j}^{n} \rightarrow 0$, a.s.  As such we would prove all categories but one become extinct.  In order to prove this result, we require a few lemmas.

\begin{lemma} Let $Z_{j}^{n}=W_{j}^{n} / W_{\mbox{\textit{tot}}}^{n}$.  If the number of exemplars per category stored is $N=\infty$, then the random variable $Z_{j}^{n}$ is a martingale with respect to the filtration $\{\mathcal{F}_{n} \}_{n\geq 1}$.
\label{martininf}
\end{lemma}

\begin{proof}
We know that $Z_{j}^{n}$ is $\mathcal{F}_{n}$-measurable \cite[p.68]{billingsley1995}, and $\expect (|Z_{j}^{n}|)\leq 1$.  Due to the fact that $Z_{j}^{n+1}$ conditioned on $\mathcal{F}_{n}$ only depends on the values of $W_{j}^{n}$, and $W_{\mbox{\textit{tot}}}^{n}$, we obtain:
\begin{align*}
\expect(Z_{j}^{n+1}|\mathcal{F}_{n})&=\expect(Z_{j}^{n+1}|W_{j}^{n},W_{\mbox{\textit{tot}}}^{n}) \\
&=\dfrac{W_{j}^{n}}{W_{\mbox{\textit{tot}}}^{n}}\left(\dfrac{W_{j}^{n} \beta  +1 }{ W_{\mbox{\textit{tot}}}^{n}  \beta  +1 }\right) 
+\dfrac{W_{\mbox{\textit{tot}}}^{n}-W_{j}^{n}}{W_{\mbox{\textit{tot}}}^{n}}\left(\dfrac{W_{j}^{n} \beta   }{W_{\mbox{\textit{tot}}}^{n} \beta  +1 }\right) \\
&=\dfrac{W_{j}^{n}}{W_{\mbox{\textit{tot}}}^{n}}=Z_{j}^{n},
\end{align*}
implying that $Z_{j}^{n}$ is a martingale with respect to the filtration $\{\mathcal{F}_{n} \}_{n\geq 1}$ \cite[p.458]{billingsley1995}. \end{proof}



\begin{lemma} There exists a $\gamma \in \R$ depending only on $\lambda$ and the initial total weight $W_{\mbox{\textit{tot}}}^{0}$, such that $W_{\mbox{\textit{tot}}}^{n}\leq \gamma$, for $i=1,2,\dots , k$, and for all $n\geq 0$.
\label{lem:weightbound}
\end{lemma}
\begin{proof}
We know that $W_{\mbox{\textit{tot}}}^{n}=\sum_{i=1}^{k}=W_{\mbox{\textit{tot}}}^{n-1}e^{-\lambda}+1$, for all realizations.  Since $e^{-\lambda} <1$, we know that $W_{\mbox{\textit{tot}}}^{n}$ converges to $W:= (1-e^{-\lambda})^{-1}$.
Since $W_{\mbox{\textit{tot}}}^{n}$ converges monotonically to W,
\beq
W_{\mbox{\textit{tot}}}^{n}\leq \max \left\{ W_{\mbox{\textit{tot}}}^{0}, \dfrac{1}{1-e^{-\lambda}} \right\}=\gamma,
\eeq
for all $n$.
This in turn implies the result.
\end{proof}

Using Lemmas~\ref{martininf}~and~\ref{lem:weightbound}, and the Martingale Convergence Theorem \cite[p.468]{billingsley1995}, we are able to prove Theorem~\ref{infcollapse}.

\begin{thm} If $N=\infty$, then for all $j\in \{1,\hdots ,k \}$, $Z_{j}^{n}$ converges a.s. to a random variable $Z_{j}^{*}$, a.s.  Furthermore, the only values $Z_{j}$ can be with positive probability are $0$ and $1$.
\label{infcollapse}
\end{thm}

\begin{proof}

To prove this theorem, we will require an expression for $\Var (Z_{j}^{n+1}|\mathcal{F}_{n})$, where $\mathcal{F}_{n}=\sigma (w_{j}^{m},~m\leq n,~j\in\{1,\hdots,k\})$, as in Lemma \ref{martininf}.  First we determine $\expect((Z_{j}^{n+1})^{2}|\mathcal{F}_{n})$,
\begin{align*}
\expect((Z_{j}^{n+1})^{2}|\mathcal{F}_{n})&=\dfrac{W_{j}^{n}}{W_{\mbox{\textit{tot}}}^{n}}\left(\dfrac{W_{j}^{n} \beta  +1 }{W_{\mbox{\textit{tot}}}^{n} \beta  +1 }\right)^{2} +\dfrac{W_{\mbox{\textit{tot}}}^{n}-W_{j}^{n}}{W_{\mbox{\textit{tot}}}^{n}}\left(\dfrac{W_{j}^{n} \beta   }{W_{\mbox{\textit{tot}}}^{n} \beta  +1 }\right)^{2} \\
&=\dfrac{(W_{j}^{n})^{2}W_{\mbox{\textit{tot}}}^{n}\beta^{2}+2(W_{j}^{n})^{2} \beta + W_{j}^{n} }{ W_{\mbox{\textit{tot}}}^{n} \left(W_{\mbox{\textit{tot}}}^{n} \beta  +1\right)^{2}  }.
\end{align*}
This allows us to calculate the conditional variance,
\begin{align*}
\Var(Z_{j}^{n+1}|\mathcal{F}_{n})&:=\expect((Z_{j}^{n+1})^{2}|\mathcal{F}_{n})-\expect(Z_{j}^{n+1}|\mathcal{F}_{n})^{2} \\
&=\dfrac{(W_{j}^{n})^{2}W_{\mbox{\textit{tot}}}^{n}\beta^{2}+2(W_{j}^{n})^{2} \beta + W_{j}^{n} }{ W_{\mbox{\textit{tot}}}^{n} \left(W_{\mbox{\textit{tot}}}^{n} \beta  +1\right)^{2}  }-
\left(\dfrac{W_{j}^{n}}{ W_{\mbox{\textit{tot}}}^{n} }\right)^{2} \\
&=W_{j}^{n}(W_{\mbox{\textit{tot}}}^{n}-W_{j}^{n})   (W_{\mbox{\textit{tot}}}^{n})^{-2}  \left(W_{\mbox{\textit{tot}}}^{n} \beta  +1\right)^{-2} \\
&=Z_{j}^{n}(1-Z_{j}^{n}) \left(W_{\mbox{\textit{tot}}}^{n} \beta  +1 \right)^{-2}. \numberthis \label{variance}
\end{align*}

By the Martingale Convergence Theorem, because $Z_{j}^{n}$ is a submartingale and $\sup_{n} \expect |Z_{j}^{n}|\leq 1$, we know there is a random variable $Z_{j}^{*}$, such that $Z_{j}^{n}\rightarrow Z_{j}^{*}$ a.s.  This implies $Z_{j}^{n+1}-Z_{j}^{n}\rightarrow 0$ a.s., and we know that $|Z_{j}^{n+1}-Z_{j}^{n}|\leq 2$ for all $n$.  By the Dominated Convergence Theorem \cite{rosenthal}, this implies that $\expect(|Z_{j}^{n+1}-Z_{j}^{n}|^{2})\rightarrow 0$, as $n\rightarrow \infty$.

By Lemma \ref{lem:weightbound}, $W_{tot}^{n}\leq \gamma$, for all $n$.  Because $Z_{j}^{n}$ is $\mathcal{F}_{n}$-measurable, and $\expect(Z_{j}^{n+1}|\mathcal{F}_{n})=Z_{j}^{n}$,
\begin{align*}
\expect(|Z_{j}^{n+1}-Z_{j}^{n}|^{2})&=\expect((Z_{j}^{n+1})^{2}-2Z_{j}^{n+1}Z_{j}^{n}+(Z_{j}^{n})^{2}) \\
&=\expect \left[ \expect((Z_{j}^{n+1})^{2}-2Z_{j}^{n+1}Z_{j}^{n}+(Z_{j}^{n})^{2} | \mathcal{F}_{n} ) \right] \\
&=\expect \left[ \Var(Z_{j}^{n+1}|\mathcal{F}_{n})\right]. \numberthis \label{vareq}
\end{align*}
Using Equations \ref{variance} and \ref{vareq}, as well as Lemma \ref{lem:weightbound}, we get the following
\begin{align*}
\expect(|Z_{j}^{n+1}-Z_{j}^{n}|^{2})
&=\expect \left[  Z_{j}^{n}(1-Z_{j}^{n}) \left(W_{\mbox{\textit{tot}}}^{n} \beta  +1 \right)^{-2}  \right] \\
&\geq (\gamma \beta +1)^{-2} \expect \left[  Z_{j}^{n}(1-Z_{j}^{n}) \right].
\end{align*}
Taking the limit as $n\rightarrow \infty$ on both sides, we obtain $\expect \left[  Z_{j}^{n}(1-Z_{j}^{n}) \right]\rightarrow 0$, as $n\rightarrow \infty$.  Because convergence in $L_{1}$ implies convergence in probability \cite[p.85]{romano}, we know $\pr (Z_{j}^{n}(1-Z_{j}^{n})<\eps)\rightarrow 1$, for all $\eps>0$.  This implies there exists a subsequence such that $Z_{j}^{n_{i}}(1-Z_{j}^{n_{i}})\rightarrow 0$ a.s. \cite[p.7]{bhattbook}.  As such $Z_{j}^{*}$ can only equal $0$ or $1$, since we know there must exist a $Z_{j}^{*}$ such that $Z_{j}^{n}\rightarrow Z_{j}^{*}$, a.s. \end{proof}

Which brings us to our final result.

\begin{thm}
When $N=\infty$, in the model described in Section \ref{sec:general}, all categories but one will become extinct with a probability of $1$.
\end{thm}
\begin{proof}
We know by Lemma~\ref{lem:weightbound} that ${Z_{j}^{n}=W_{j}^{n} / W_{\mbox{\textit{tot}}}^{n} \geq W_{j}^{n} \gamma^{-1} \geq 0}$, implying if $Z_{j}^{n}\rightarrow 0$, then $W_{j}^{n}\rightarrow 0$ as well.  By Theorem \ref{infcollapse}, for every category ${j\in \{1\hdots k \}}$, $Z_{j}^{n}\rightarrow Z_{j}^{*}$, a.s., where $Z_{j}^{*}$ can only be $0$ or $1$, and we know $\sum_{j} Z_{j}^{*}=1$.  As such $Z_{j}^{n}\rightarrow 0$, a.s, for every category $j$, but one.  This implies all but one category will become extinct with a probability of $1$.
\end{proof}

\section{Simulations and Time to Extinction}
\label{simstoptime}

In the last two sections, we proved extinction of all but one category occurs for our model regardless of the value of $N$.  In this section we will discuss some of the results obtained by computer simulations of the simplified weight model.  These simulations will demonstrate how changing the variables $N$ and $\lambda$ affects how long it takes until there is only one non-extinct category left in the system.

\begin{figure}[h!]
\includegraphics[width=14cm]{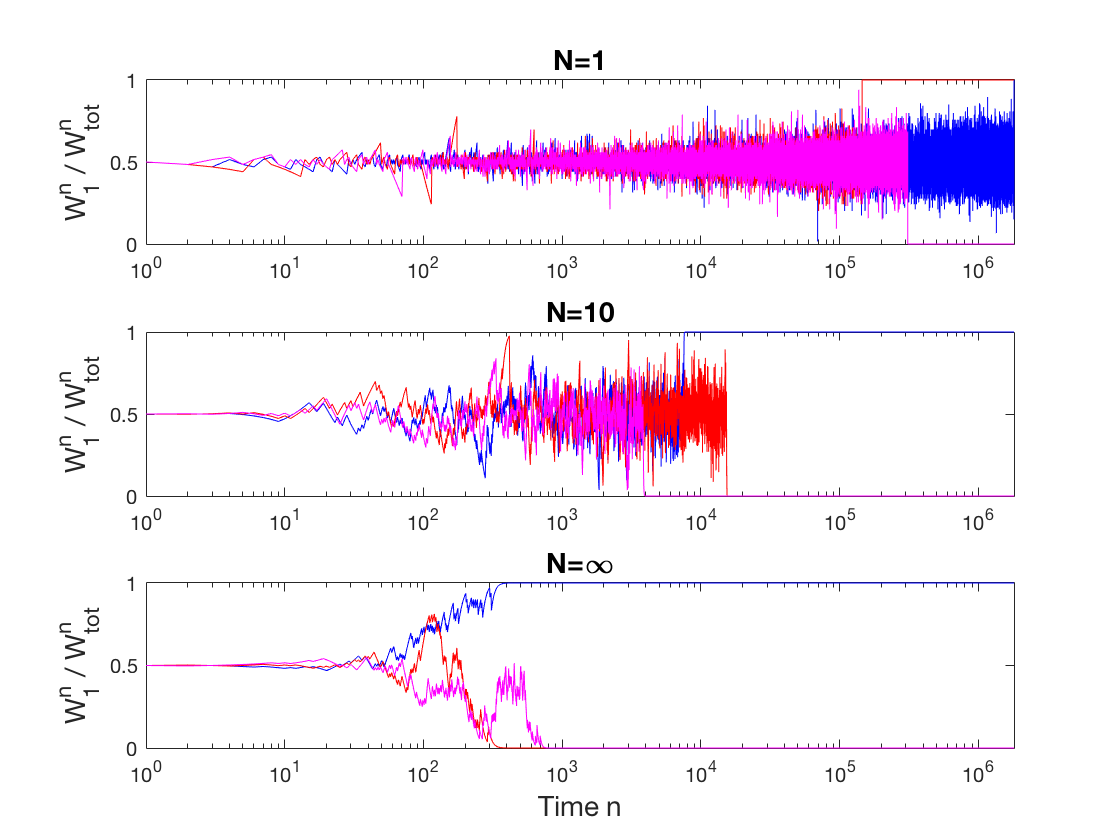}
	\caption{Plots of $Z_{1}^{n}=W_{1}^{n}/W_{\mbox{\textit{tot}}}^{n}$ for single simulations when $k=2$, $\lambda=0.06$, and the weight threshold is $10^{-4}W_{0}$.  For each value of $N$ we have plotted $Z_{1}^{n}$ against time step $n$ for three simulations.}
	\label{fig:wsim}
\end{figure}

\begin{figure}[h!]
\includegraphics[width=14cm]{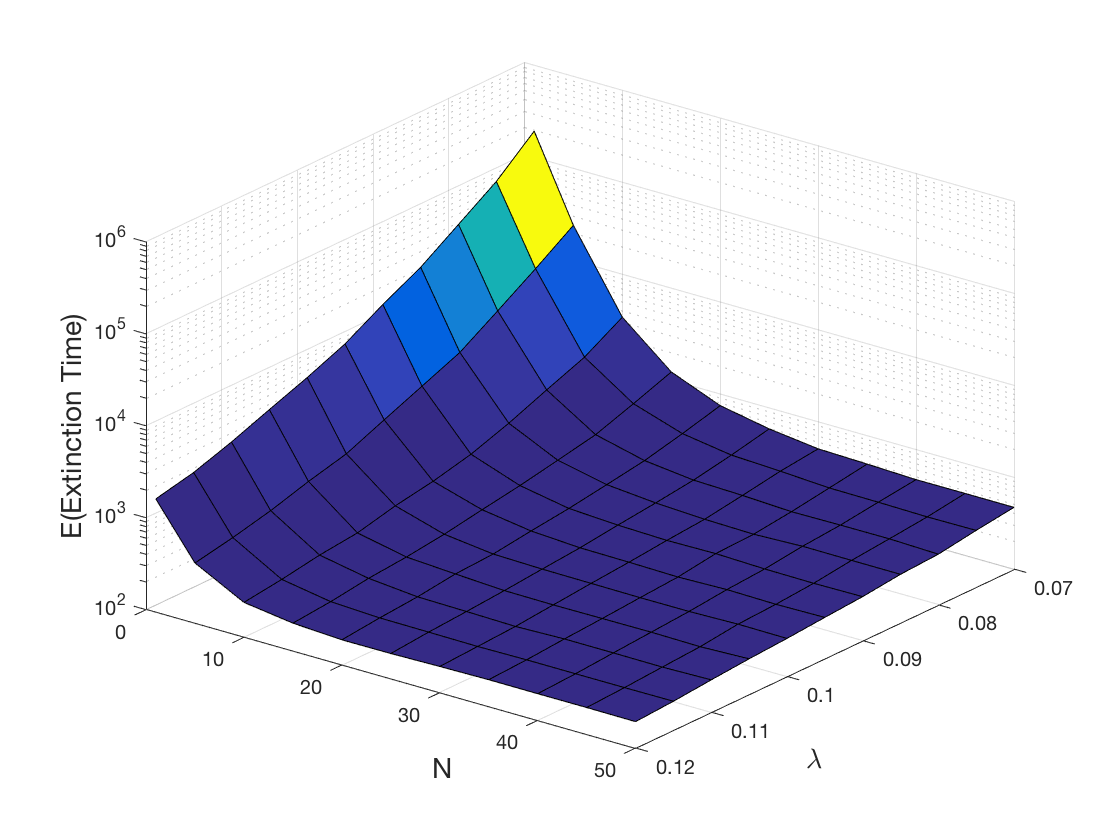}
	\caption{Plotting the expected extinction time as we change variables $N$ and $\lambda$.  We use a weight threshold equal to $10^{-4}W_{0}$.}
	\label{fig:stoptime}
\end{figure}

Before discussing the results of our computer simulations, we will explain weight thresholds.  Analytically a category~$j$ becomes extinct when $W_{j}^{n}\rightarrow 0$, as $n\rightarrow \infty$.  Extinction of all but one category means there exists a category~$j$ such that $W_{i}^{n}\rightarrow 0$, as $n\rightarrow \infty$, for all $i\neq j$.  When running computer simulations, we cannot possibly know for certain if a category's weight approaches zero, but we do something else to detect if it is most likely going to.  In simulations, once a category's weight goes below a value we call a weight threshold, we assume that the category becomes extinct.  The time it takes for all but one of the category's weights to go below the weight threshold will be referred to as the extinction time.  For Figures \ref{fig:wsim} and \ref{fig:stoptime}, the number of categories $k=2$, so the extinction time is how soon one of the two categories goes extinct.



Figure \ref{fig:wsim} plots three simulations each for three separate values of $N$, where the number of categories $k=2$.  We show the evolution of the random variable $Z_{1}^{n}$ (defined in Section \ref{infstored}) for the values $N=1,~10$ and $\infty$.  When $Z_{1}^{n}$ hits either $0$ or $1$ the simulation ends, representing that either category $1$ or $2$ has become extinct respectively.  Upon inspection we see that the larger $N$ is, the faster categories become extinct.

Figure \ref{fig:stoptime} plots how the expected extinction time changes based on the values of our decay rate $\lambda$, and the limitation on the number of exemplars $N$, when the number of categories $k=2$.  The expected value for the extinction time is found by averaging over $1000$ simulations for each value of $N$ and $\lambda$.  As $N$ decreases, we observe as we did for Figure \ref{fig:wsim} that the extinction time increases.  Likewise as $\lambda$ decreases, the extinction time increases as well.

It is straightforward to explain how $\lambda$ affects the extinction time, but the explanation for the effect $N$ has is more subtle.  To help understand the effect $N$ has on the extinction time, we will consider two examples.  For both examples, let $ \beta =0.5$, $k=2$ (two categories), and the initial weight of the first two exemplars in each list is $W_{0}=1$, while the rest of the exemplar weights are zero.

\begin{enumerate}
\item First consider the case where $N=2$.  If $x_{0}=2$, the weights of category $2$ will be $w_{2,1}^{2}=1$, and $w_{2,2}^{2}=0.5$, and the weights of category $1$ will be $w_{1,1}^{2}=w_{1,2}^{2}=\beta = 0.5$.  This implies the probability that $x_{1}=2$, given that $x_{0}=2$, is $60\%$.

\item Now consider the case where $N=\infty$.  If $x_{0}=2$, then the total weight of categories $1$ and $2$ respectively will be $W_{1}^{n}=2\beta=1$, and $W_{2}^{n}=2\beta+1=2$.  This implies the probability that $x_{1}=2$, given that $x_{0}=2$, is approximately $66.7\%$.


\end{enumerate}

It is more probable when $N=\infty$, for a category to be consecutively categorized.  When $N=2$, it is rarer for the exemplar weights to decay close to zero than when $N=\infty$.  This demonstrates why limiting the number of exemplars makes extinction take longer.  When $N=\infty$, exemplars getting stored in a category consecutively adds comparatively more weight to the category.  This explains the effect of $N$ on the extinction time, as seen in Figures \ref{fig:wsim} and \ref{fig:stoptime}.


The behaviour of  the expected extinction time increasing as $\lambda$ decreases is much easier to explain.  The weights are decaying slower, so it will take longer for the weights to approach zero.  If $\lambda =0$, then there would be no decay and thus no category extinction.  Because of this, as $\lambda \rightarrow 0$, the expected extinction time will asymptotically approach infinity.

\section{Discussion}

The model studied in this paper is simpler than the ones studied by Tupper \cite{tupper2015} and Wedel \cite{wedel2012}, but it helps explain the behaviour we see in these models.  Changing $N$ in our model doesn't affect whether all categories but one eventually become extinct, but it does affect the time it takes to do so.  Our results agree with the extinction result demonstrated in \cite{tupper2015} for $N=\infty$.  However, our work suggests that the model studied in \cite{wedel2012} will eventually show the same behaviour but on a longer time scale.  This longer time scale likely explains why category extinction was not observed in Wedel's simulation \cite{wedel2012}.

One natural direction we can take in future research is to apply our model to real world data.    For example, in Section \ref{sec:introduction}, Figure \ref{fig:cider} shows the evolution of the usage of two spellings of the word cider over $200$ years \cite{ngramviewer}.  The archaic spelling ``cyder'' becomes extinct close to the year $1980$.  Using the corpus of digitized texts put together in \cite{ngramviewer}, one could determine what values of $N$ and $\lambda$ best models this type of data.





{\bf Acknowledgments.} This work was supported by an NSERC Discovery Grant, an NSERC Discovery Accelerator Supplement, and a Tier 2 Canada Research Chair.


\bibliographystyle{apt}      

\def\cprime{$'$} \def\cprime{$'$}

%
%

\end{document}